\documentclass[sigconf,natbib=true]{acmart}

\AtBeginDocument{%
  \providecommand\BibTeX{{%
    \normalfont B\kern-0.5em{\scshape i\kern-0.25em b}\kern-0.8em\TeX}}}

\copyrightyear{2024}
\acmYear{2024}
\setcopyright{acmlicensed}\acmConference[SIGIR '24]{Proceedings of the 47th International ACM SIGIR Conference on Research and Development in Information Retrieval}{July 14--18, 2024}{Washington, DC, USA} \acmBooktitle{Proceedings of the 47th International ACM SIGIR Conference on Research and Development in Information Retrieval (SIGIR '24), July 14--18, 2024, Washington, DC, USA}
\acmDOI{10.1145/3626772.3657721} 
\acmISBN{979-8-4007-0431-4/24/07}


\usepackage{subfigure}

\usepackage{multirow}
\usepackage{multicol}
\newtheorem{problem definition}{Problem Definition}
\newtheorem{theorem}{Theorem}
\usepackage{enumitem}

\begin{document}

\title{TransGNN: Harnessing the Collaborative Power of Transformers and Graph Neural Networks for Recommender Systems}

\author{Peiyan Zhang}
\authornote{Both authors contributed equally to this research.}
\affiliation{\institution{Hong Kong University of \\ Science and Technology}\country{Hong Kong}}
\email{pzhangao@cse.ust.hk}

\author{Yuchen Yan}
\authornotemark[1]
\affiliation{\institution{School of Intelligence Science and Technology, Peking University}\city{Beijing}\country{China}}
\email{2001213110@stu.pku.edu.cn}

\author{Xi Zhang}
\affiliation{\institution{Interdisciplinary Institute for Medical Engineering, Fuzhou University}\country{China}}
\email{zxwinner@gmail.com}

\author{Chaozhuo Li}
\authornote{Chaozhuo Li is the corresponding author}
\affiliation{\institution{Microsoft Research Asia}\city{Beijing}\country{China}}
\email{lichaozhuo1991@gmail.com}

\author{Senzhang Wang}
\affiliation{\institution{Central South University}\country{China}}
\email{szwang@csu.edu.cn}

\author{Feiran Huang}
\affiliation{\institution{Jinan University}\country{China}}
\email{huangfr@jnu.edu.cn}


\author{Sunghun Kim}
\affiliation{\institution{Hong Kong University of \\ Science and Technology}\country{Hong Kong}}
\email{hunkim@cse.ust.hk}

\renewcommand{\shortauthors}{Peiyan Zhang et al.}
\begin{abstract}
  Graph Neural Networks (GNNs) have emerged as promising solutions for collaborative filtering (CF) through the modeling of user-item interaction graphs.
  The nucleus of existing GNN-based recommender systems involves  recursive message passing along user-item interaction edges to refine encoded embeddings. 
  Despite their demonstrated effectiveness, current GNN-based methods encounter challenges of limited receptive fields and the presence of noisy "interest-irrelevant" connections.  
  In contrast, Transformer-based methods excel in aggregating information adaptively and globally. 
  Nevertheless, their application to large-scale interaction graphs is hindered by inherent complexities and challenges in capturing intricate, entangled structural information.  
  In this paper, we propose TransGNN, a novel model that integrates Transformer and GNN layers in an alternating fashion to mutually enhance their capabilities. 
  Specifically, TransGNN leverages Transformer layers to broaden the receptive field and disentangle information aggregation from edges, which aggregates information from more relevant nodes, thereby enhancing the message passing of GNNs. 
  Additionally, to capture graph structure information effectively, positional encoding is meticulously designed and integrated into  GNN layers to encode such structural knowledge into node attributes, thus enhancing the Transformer's performance on graphs.
  Efficiency considerations are also alleviated by proposing the sampling of the most relevant nodes for the Transformer, along with two efficient sample update strategies to reduce complexity. Furthermore, theoretical analysis demonstrates that TransGNN offers increased expressiveness compared to GNNs, with only a marginal increase in linear complexity. Extensive experiments on five public datasets validate the effectiveness and efficiency of TransGNN. Our code is available at https://github.com/Peiyance/TransGNN-torch.
\end{abstract}




\maketitle

\section{Introduction}
\par Recommender systems play vital roles in various online platforms, due to their success in addressing information overload challenges by recommending useful content to users~\citep{zhang2023efficiently,guo2022evolutionary, zhou2023exploring,liu2023chatgpt}. To accurately infer the user preference, encoding user and item informative representations is the core part of effective collaborative filtering (CF) paradigms based on the observed user-item interactions~\citep{he2017neural,rendle2020neural}. Recent years have witnessed a proliferation of development of graph neural networks (GNNs) for modeling graph-structural data~\citep{li2019adversarial,yang2021graphformers,li2017ppne,yang2022semantic,zhao2022learning}. One promising direction is to perform the information propagation along the user-item interactions to refine user embeddings based on the recursive aggregation schema~\citep{ying2018graph,wang2019neural,he2020lightgcn}. 


\par Notwithstanding the effectiveness of the existing graph-based CF models, several fundamental challenges remain inadequately resolved. 
\textit{First}, the message passing mechanism relies on edges to fuse the graph structures and node attributes, leading to strong bias and potential noise~\citep{BalcilarHGVAH21}. For example, recent studies on eye tracking demonstrate that users are less likely to browse items that are ranked lower in the recommended lists, while they tend to interact with the first few items at the top of lists, regardless of the items’ actual relevance~\citep{joachims2007evaluating,joachims2017accurately}. 
Hence, the topological connections within interaction graphs are impeded by the aforementioned positional bias, resulting in less convincing message passings~\citep{chen2023bias}. 
Besides,  users may interact with products they are not interested in due to the over-recommendation of popular items \citep{zhang2021causal}, leading to the formation of "interest-irrelevant connections" in the user-item interaction graph.  
As such, the graph generated from user feedback towards the recommended lists may fail to reflect user preference faithfully~\citep{collins2018study}. 
Worse still, the propagation of embeddings along edges can exacerbate noise effects, potentially distorting the encoding of underlying user interests in GNN-based models.

\par \textit{Second}, the receptive field of GNNs is also constrained by the challenge of over-smoothing \citep{deeper_insight}.
It has been proven that as the GNNs architecture goes deeper and reaches a certain extent, the model will not respond to the training data, and the node representations obtained by such deep models tend to be over-smoothed and also become indistinguishable~\citep{AnoteOnOversmooth, GPRGNN, Li0TG19, Uri2Y21}. 
Consequently, the optimal number of layers for GNN models is typically limited to no more than 3 \citep{ying2018graph, wang2019neural, he2020lightgcn}, where the models can only capture up to 3-hop relations. 
However, in real world applications, item sequences often exceed a length of 3, suggesting the presence of important sequential patterns that extend beyond this limitation. Due to the inherent constraint of the network structure, GNN-based models struggle to capture such longer-term sequential information. 

\par Fortunately, the Transformer architecture~\citep{attention} appears to provide an avenue for addressing these inherent limitations. Owing to the self-attention mechanism, every items can aggregate information from all the items in the user-item interaction sequence. Consequently, Transformer can capture the long-term dependency within the sequence data, and has displaced the convolutional and recurrent neural networks to become the new de-facto standard among many recommendation tasks~\citep{jiang2022adamct,fan2021continuous}. 
Nevertheless, while Transformers exhibit the capability to globally and adaptively aggregate information, their ability to effectively utilize graph structure information is constrained. This limitation  stems from the fact that the aggregation process in Transformers does not rely on edges, resulting in an underestimation of crucial historical interactions~\citep{min2022transformer}.

\par In this paper, we inquire whether the integration of Transformers and GNNs can leverage their respective strengths to mutually enhance performance. By leveraging Transformers, the receptive field of GNNs can be expanded to encompass more relevant nodes, even those located distantly from central nodes. Conversely, GNNs can assist Transformers in capturing intricate graph topology information and efficiently aggregating relevant nodes from neighboring regions. 
Nevertheless, the integration of GNNs and Transformers for modeling graph-structured CF data poses significant challenges, primarily encompassing the following three core aspects. 
(1) \textit{How to sample the most relevant nodes in the attention sampling module?} As the user-item interaction graph may contain ``interest-irrelevant'' connections, directly aggregating information from all interaction edges will impair the accurate user representation. Meanwhile, considering the most relevant nodes not only reduces computational complexity but also filters out irrelevant information from noisy nodes.
(2) \textit{How can Transformers and GNNs be effectively coupled in a collaborative framework?} 
Given the individual merits inherent to both Transformers and GNNs, it posits a logical progression to envisage a collaborative framework where these two modules engage in a mutual reinforcement within user modeling. (3) \textit{How to update the attention samples efficiently to avoid exhausting complexity?} 
The computation of self-attention weights across the entire graph dataset for each central node entails a time and space complexity of $O(N^2)$, posing challenges such as the out-of-memory problem with increasingly large graphs. Hence, there exists an imperative to devise efficient strategies for updating attention samples.

\par To tackle aforementioned challenges, we introduce a novel framework named \textbf{TransGNN}, which amalgamates the prowess of both GNNs and Transformers. 
To mitigate complexity and alleviate the influence of irrelevant nodes, we 
 first propose sampling attention nodes for each central node based on semantic and structural information.
After that,  we introduce three types of positional encoding: (i) shortest-path-based positional encoding, (ii) degree-based positional encoding, and (iii) PageRank-based positional encoding. 
Such positional encoding embed various granularity of  structural topology information into node embeddings, facilitating the extraction of simplified graph structure information for Transformers. 
Then, we devise the TransGNN module where Transformers and GNNs alternate to mutually enhance their performance. 
Within the GNN layer, Transformers aggregate attention sample information with low complexity to expand GNNs' receptive fields focusing on the most relevant nodes.
Conversely, within Transformers, GNNs' message-passing mechanism aids in fusing representations and graph structure to capture rich topology information. 
Efficient retrieval of more relevant node information from neighborhoods is also facilitated by the message-passing process.  
Furthermore, we propose two efficient methods for updating attention samples, which can easily be generalized to large-scale graphs. 
Finally, a theoretical analysis of the expressive ability and complexity of TransGNN is presented, demonstrating its enhanced expressiveness compared to GNNs with only marginal additional linear complexity. 
TransGNN is extensively evaluated over five public benchmarks, and  substantial improvements demonstrate its superiority. 

Our contributions can be summarized as follows:
\begin{itemize}[leftmargin=*,noitemsep,topsep=0pt]
    \item We introduce a novel model, TransGNN, wherein Transformers and GNNs synergistically collaborate. Transformers broaden the receptive field of GNNs, while GNNs capture essential structural information to enhance the Transformer's performance.
    \item To mitigate the challenge of complexity, we introduce a sampling strategy along with two efficient methods for updating relevant samples efficiently.  
    \item We perform a theoretical analysis on the expressive capacity and computational complexity of TransGNN, revealing that TransGNN exhibits greater potency compared to GNNs with small additional computational overhead. 
    \item We conduct comprehensive experiments on five public datasets from different domains, where TransGNN outperforms competitive baseline models significantly and consistently. In-depth analysis are provided towards the rationality of TransGNN from both technical and empirical perspectives. 
\end{itemize}

\section{Related Work}
\textbf{Recap Graph Collaborative Filtering Paradigm.} Graph-based collaborative filtering paradigm introduce graph structures to represent the interactions between users and items. Given \textit{I} users and \textit{J} items with the user set $\mathcal{U}=\{u_{1},...,u_{\textit{I}}\}$ and item set $\mathcal{V}=\{v_{1},...,v_{\textit{J}}\}$, edges in the user-item interaction graph $\mathcal{G}$ are constructed if user $u_{i}$ has interacted with item $v_{j}$. Through the incorporation of the user-item interaction graphs, graph-based CF methods manage to capture multi-order connectivity information, leading to more accurate recommendation outcomes.


\noindent \textbf{Recommendation with Graph Neural Networks.} Recent works have embarked on formulating diverse  graph neural architectures to model the intricate user-item interaction landscapes via embedding propagation. Using the message passing schema, both users and items are transformed into embeddings that retain the information from multi-hop connections. 
Notably, PinSage~\cite{ying2018graph} and NGCF~\cite{wang2019neural} anchor their foundations on the graph convolutional framework within the realm of the spectral domain. Subsequently, LightGCN~\cite{he2020lightgcn} advocates for a more streamlined approach by sidelining intricate non-linear transformations and endorsing the sum-based pooling applied to neighboring representations.
Although GNNs have achieved state-of-the-art performance in CF, the limited receptive field compromises their power. Shallow GNNs can only aggregate nearby information, which shows strong structure bias and noise, while deep GNNs suffer from the over-smoothing problem and aggregate much irrelevant information~\cite{defsubspace}.

\noindent \textbf{Recommendation with Transformers.} Recently, attentive modules are extensively studied in the recommendation venue, resulting in extraordinary performances~\cite{kang2019recommender,jiang2022adamct,zhou2023attention}. Self-attention models, in particular, have garnered substantial attention due to their capacity for point-to-point feature interactions within item sequences. This mechanism effectively addresses the challenge of global dependencies and enables the incorporation of longer sequences enriched with a wealth of information~\cite{kang2018self,luo2021stan,jin2022code}. Many existing works exert their efforts to generalize the Transformer architecture  to graph data. However the main problems they encounter are: (1) The design of node-wise positional encoding. (2) The computational expensive calculation of pairwise attention on large graphs. For positional encoding, Laplacian Encoding~\cite{dwivedi2020generalization}, and Random Walk have been studied both theoretically and empirically. With respect to the scalability problem, some work try to restrict the receptive filed from global to local, for example, ADSF~\cite{Zhang2020Adaptive} introduces random walks to generate high order local receptive filed, and the GAT~\cite{gat} is the extreme scenario where each node only sees its one-hop neighbor. 

\begin{figure*}[h]
\centering
\includegraphics[width=0.99\textwidth]{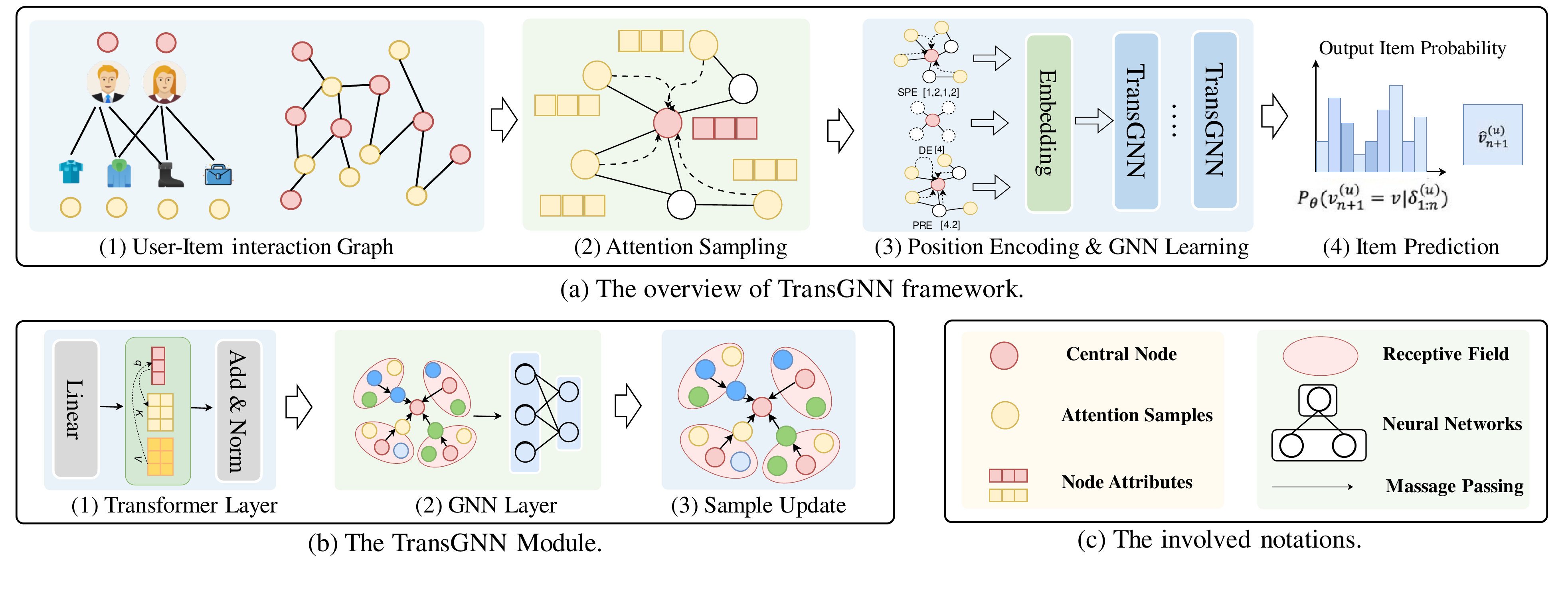}
\caption{The framework of TransGNN. We first sample relevant nodes for the central nodes, then we calculate positional encoding to enhance the raw attributes by combining the structure information. In the TransGNN module, the Transformer layer and GNN layer improve each other, followed by the samples update sub-module.}
\label{figure:framework}
\end{figure*}
\section{Methodology}
\par This section begins with an exposition of the TransGNN framework, followed by a detailed elucidation of each constituent component. 
\subsection{Model Framework}
\par The framework of TransGNN is shown in \autoref{figure:framework}, which consists of three important components: (1) attention sampling module, (2) positional encoding module, (3) TransGNN module. 
We first sample the most relevant nodes for each central node by considering the semantic similarity and graph structure information in the attention sampling module. Then in the positional encoding module, we calculate the positional encoding to help the Transformer capture the graph topology information.  After these two modules, we use the TransGNN module, which contains three sub-modules in order: (i) Transformer layer, (ii) GNN layer, (iii) samples update sub-module. Among them, the Transformer layer is used to expand the receptive field of the GNN layer and aggregate the attention samples information efficiently, while the GNN layer helps the Transformer layer perceive the graph structure information and obtain more relevant information of neighbor nodes. The samples update sub-module is integrated to update the attention samples efficiently.

\subsection{Attention Sampling Module}
\par Calculating attention across the entire user-item interaction graph presents two notable challenges: 
(i) The computational complexity of attention calculation scales quadratically ($O(N^2)$), which becomes impractical for large-scale recommender systems. 
(ii) Under the global attention setting, irrelevant user-item interactions are also incorporated, leading to suboptimal performance.

\par In the context of recommender systems, we posit that it is unnecessary to compute attention across the entire graph for each node. Instead, prioritizing the most relevant nodes is sufficient, thereby reducing computational complexity and eliminating noisy node information. Consequently, we advocate for sampling the most pertinent nodes for a given user or item node within the attention sampling module. To facilitate this, we commence by computing the semantic similarity matrix: 
\begin{equation}
    \mathbf S = \mathbf X \mathbf X^{\top},
    \label{equ:first_similarity}
\end{equation}
where $\mathbf X \in \mathbb R^{N\times d}$ consists of nodes' attributes. However, through $\mathbf S$ we can only get the raw semantic similarity, overlooking the structural intricacies of user preferences. Recognizing that a user's preference for one item might influence their affinity for another (due to shared attributes or latent factors), we refine our similarity measure by considering the neighbor nodes' preference before sampling.
We use the following equation to update the similarity matrix to  incorporate the preferences of neighboring nodes: 
\begin{equation}
    \mathbf S = \mathbf S + \alpha \hat{\mathbf A} \mathbf S,
    \label{equ:second_similarity}
\end{equation}
where $\alpha$ is the balance factor, and in this paper, we set $\alpha$ as 0.5. $\hat{\mathbf A}=\mathbf A+\mathbf I$ where $\mathbf A \in \mathbb R^{N\times N}$ is the adjacent matrix and $\mathbf I$ is the identity matrix. Based on the new similarity matrix $\mathbf S \in \mathbb R^{N\times N}$, for every node $v_i \in \mathcal V$ in the input graph, we sample the most relevant nodes as its attention samples as follows:

\textbf{Attention samples}: Given an input graph $\mathcal G$ and its similarity matrix $\mathbf S$, for node $v_i$ in the graph, we define its attention samples as set $\text{Smp}(v_i)=\left\{v_j| v_j \in V \, \text{and} \, S(i,j) \in \text{top-k}(S(i,:))\right\}$ where $S(i,:)$ denotes the $i$-th row of $\mathbf S$ and the $k$ works as a hyper-parameter which decides how many nodes should be attended attention. 


\subsection{Positional Encoding Module}
\par User-item interactions in recommender systems embody intricate structural information, critical for deriving personalized recommendations. Unlike the grid-like data where the sequential patterns can be easily captured by Transformer, interaction graphs present a more challenging topology to navigate. To enrich Transformers with this topological knowledge, we introduce three distinct positional encodings tailored for recommendation scenarios: (i) Shortest path hop based positional encoding. (ii) Degree-based positional encoding. (iii) PageRank based positional encoding. 
The first two encoding signify the proximity between users and items, emphasizing the diversity and frequency of user interactions or the popularity of items. Meanwhile, the last encoding indicates the significance determined by the graph topology.

\subsubsection{Shortest Path Hop based Positional Encoding}
\par User-item proximity in interaction graphs can hint at user preferences. For every user, the distance to various items (or vice versa) can have distinct implications. We encapsulate this by leveraging shortest path hops.
Specifically, we denote the shortest path hop matrix as $\mathbf P$ and for each node $v_i \in \mathcal V$ and its attention sample node $v_j \in \text{Smp}(v_i)$ the shortest path hop is $P(i, j)$, we calculate the shortest path hop based positional encoding(SPE) for every attention sample node $v_j$ as:
\begin{equation}
    \text{SPE}(v_i, v_j) = \text{MLP}(P(i, j)),
\end{equation}
where $\text{MLP}(\cdot)$ is implemented as a two-layer neural network.

\subsubsection{Degree based Positional Encoding}
\par The interaction frequency of a user, or the popularity of an item, plays a pivotal role in recommendations. An item's popularity or a user's diverse taste can be harnessed using their node degree in the graph.
Therefore, we propose to use the degree to calculate the positional encoding. 
Formally, for any node $v_i$ whose degree is $\text{deg}_i$, we calculate the degree based positional encoding(DE) as:
\begin{equation}
    \text{DE}(v_i) = \text{MLP}(\text{deg}_i). 
\end{equation} 

\subsubsection{Page Rank based Positional Encoding}
\par Certain users or items exert more influence due to their position in the interaction graph. PageRank offers a way to gauge this influence, facilitating better recommendations.
In order to obtain the influence of structure importance, we propose to calculate positional encoding based on the page rank value for every node. Formally, for node $v_i$ we denote its page rank value as $\text{Pr}(v_i)$, and we calculate the page rank based positional encoding(PRE) as:
\begin{equation}
    \text{PRE}(v_i) = \text{MLP}(\text{Pr}(v_i)). 
\end{equation}

\par By aggregating the above encodings with raw user/item node attributes, we enrich the Transformer's understanding of the recommendation landscape.
Specifically, for the central node $v_i$ and its attention samples $\text{Smp}(v_i)$, we aggregate the positional encoding in the following ways:
\begin{equation}
\small
\begin{aligned}
    & \mathbf h_i = \text{COMB}(\text{AGG}(\mathbf x_i, \text{SPE}(v_i, v_i), \text{DE}(v_i), \text{PRE}(v_i))) \\
    & \mathbf h_j = \text{COMB}(\text{AGG}(\mathbf x_j, \text{SPE}(v_i, v_j), \text{DE}(v_j), \text{PRE}(v_j))) \, v_j \in \text{Smp}(v_i)
\end{aligned}
\end{equation}
where $\mathbf x_i, \mathbf x_j$ are the raw attributes of $v_i, v_j$ respectively, $\text{AGG}(\cdot)$ is the aggregation function and $\text{COMB}(\cdot)$ is the combination function. In this paper, we use a two-layer MLP  as $\text{COMB}(\cdot)$ and vector concatenation as $\text{AGG}(\cdot)$.

\subsection{TransGNN Module}
\par Traditional Graph Neural Networks (GNNs) exhibit limitations in comprehending the expansive relationships between users and items, due to their narrow receptive fields and the over-smoothing issue in deeper networks. Crucially, relevant items for users might be distant in the interaction space. Although Transformers can perceive long-range interactions, they often miss out on the intricacies of structured data in recommendation scenarios, further challenged by computational complexities. The TransGNN module synergizes the strengths of GNNs and Transformers to alleviate these issues. This module consists of: (i) Transformer layer, (ii) GNN layer, and (iii) samples update sub-module.


\subsubsection{Transformer Layer}
\par To optimize user-item recommendation, the Transformer layer broadens GNN's horizon, focusing on potentially important yet distant items.
In order to lower the complexity and filter out the irrelevant information, we only consider the most relevant samples for each central node. In the following, we use central node $v_i$ and its attention samples $\text{Smp}(v_i)$ as an example to illustrate the Transformer layer, and for other nodes, this process is the same.
\par We denote the input of Transformer layer as $\mathbf H \in \mathbb R^{N\times d_\text{in}}$ and the representation of central node $v_i$ is $h_i$. We stack the representations of attention samples $\text{Smp}(v_i)$ as the matrix $\mathbf H^{\text{Smp}}_i \in \mathbb R^{k\times d_\text{in}}$. We use three matrices $\mathbf W_q, \mathbf W_k, \mathbf W_v \in \mathbb R^{d_\text{in}\times d_\text{out}}$ to project the corresponding representations to $\mathbf Q$, $\mathbf K$ and $\mathbf V$ 
respectively and we aggregate the information based on the attention distribution as:

\begin{equation}
\begin{aligned}
    \mathbf{h}_i & = \text{softmax}(\frac{\mathbf Q \mathbf K^{\top}}{\sqrt{d_\text{out}}})V   \\
\end{aligned}
\label{equ:attn_agg}
\end{equation}
where
\begin{equation}
	\begin{cases}
		\mathbf Q & = \mathbf h_i \mathbf{W}_Q,\\
		\mathbf K & = \mathbf H^\text{Smp}_i \mathbf{W}_k,\\
		\mathbf V & = \mathbf H^\text{Smp}_i \mathbf{W}_v \\
	\end{cases}
\end{equation}
in which $\mathbf Q$ is the representation of the query and $\mathbf K, \mathbf V$ are the representation of keys and values. This process can be expanded to multi-head attention as:
\begin{equation}
    \text{MultiHead}(\mathbf h_i)=\text{Concat}(\text{head}_1,...,\text{head}_m)\mathbf W_m,
\end{equation}
where $m$ is the head number, $\text{Concat}(\cdot)$ denotes the concatanate function and $\mathbf W_m$ is the projection matrix, each head is calculated as $\mathbf h_i$ in \autoref{equ:attn_agg}.

\subsubsection{GNN Layer}
\par Incorporating interactions and structural nuances, this layer aids the Transformer in harnessing the user-item interaction graphs more profoundly.
Given node $v_i$, the message passing process of the GNN layer can be described as:
\begin{equation}
\begin{aligned}
    & \mathbf h_M(v_i) = \text{Message}(\mathbf h_k,\forall v_k \in \mathcal N(v_i)) \\
    & \mathbf h_i = \text{Combine}(\mathbf h_i, \mathbf h_M(v_i)),
\end{aligned}
\end{equation}
where $\mathcal N(v_i)$ is the neighbor nodes set of $v_i$. $h_i, h_k$ are the representations of $v_i, v_k$ respectively. $\text{Message}(\cdot)$ and $\text{Combine}(\cdot)$ are the message passing function and aggregation function defined by GNN layer.

\subsubsection{Samples Update Sub-Module}
\par After the Transformer and GNN layers, the attention samples should be updated upon new representations.  However, directly calculating the similarity matrix incurs a computational complexity of $O(N^2)$. Here we introduce two efficient strategies for updating the attention samples. 

\paragraph{(i) Random Walk based Update}
\par Recognizing the tendency of users to exhibit consistent taste profiles, this approach delves into the local neighborhoods of each sampled item to uncover potentially relevant items. We employ a random walk strategy to explore the local neighborhood of every sampled node. Specifically, the transition probability of the random walk is determined based on the similarity, as follows: 
\begin{equation}
    p_{i\to j} = 
  \begin{cases}
    \frac{\mathbf h_i \mathbf h_j^{T}}{\sum_{l \in \mathcal N(v_i)} \mathbf h_i \mathbf h_l^{\top}}, &\text{if $v_j \in \mathcal N(v_i)$}\\
	0, &\text{if $v_j \not\in \mathcal N(v_i)$}
  \end{cases}
\end{equation}
The transfer probability can be calculated efficiently in the message passing process. Based on the transfer probability, we walk a node sequence with length $L$ for each attention sample, and then we choose the new attention samples among the explored nodes upon new representations.
\paragraph{(ii) Message Passing based Update}
\par The random walk-based update strategy has extra overhead. We propose another update strategy that utilizes the message passing of the GNN layer to update the samples without extra overhead. Specifically, we aggregate the attention samples from the neighbor nodes for each central node in the message passing process of the GNN layer. The intuition behind this is that the attention samples of the neighbors may also be the relevant attention samples of the central nodes. We denote the attention samples set of the neighbor nodes as the attention message, which is defined as follows:
\begin{equation}
    \text{Attn\_Msg}(v_i) = \bigcup \text{Smp}(v_j) \; \forall v_j \in \mathcal N(v_i),
\end{equation}
thus we choose the new attention samples among $\text{Attn\_Msg}(v_i)$ for node $v_i$ based on the new representations.

\subsection{Model Optimization}
For training our TransGNN model, we employ the pairwise rank loss to optimize the relative ranking of items~\citep{chen2009ranking}:
\begin{equation}
    \mathcal{L}=\Sigma_{\mathcal{u}\in \mathcal{U}}\Sigma_{t=1}^{n}\log(\sigma(P(i_{t+1})-P(i_{t+1}^{-}))),
\end{equation}
where we pair each ground-truth item $i_{t+1}$ with a negative item $i_{t+1}^{-}$ that is randomly sampled. $P(i_{t+1})$ and $P(i_{t+1}^{-})$ are prediction scores given by TransGNN, $\sigma(.)$ denotes the sigmoid function.

\subsection{Complexity Analysis}
\par Here the complexity of TransGNN is presented and discussed. The overhead of the attention sampling module and the positional encoding module can be attributed to the data pre-processing stage. The complexity of the attention sampling module is $O(N^2)$, and the most complex part of the positional encoding module is the calculation of the shortest hop matrix $\mathbf P$. Considering the graphs in real applications are sparse and have positive edge weights, Johnson algorithm~\cite{allaoui2009johnson} can be employed to facilitate the analysis. With the help of the heap optimization, the time and space complexity can be reduced to $O(N(N+E)\log E)$ where $N$ is the node number, and $E$ is the edge number. The extra overhead of the TransGNN module compared with GNNs mainly focus on the Transformer layer and attention samples update. The extra complexity caused by the Transformer layer is $O(Nk)$ where $k$ is the attention samples number, and the extra complexity of samples update is $O(Nkd_a)$ for message passing mechanism where $d_a$ is the average degree (the extra complexity will be $O(NkL)$ if we use random walk-based update). \textbf{Therefore, we show that TransGNN has at most $O(N(N+E)\log E)$ data pre-processing complexity and linear extra complexity compared with GNNs because $kd_a$ is a constant and $kd_a<<N$.}

\subsection{Theoretical Analysis}
\par Here we demonstrate the expression power of TransGNN via following two theorems with their proofs. 
\begin{theorem}
TransGNN has at least the expression ability of GNN, and any GNN can be expressed by TransGNN.
\end{theorem}
\begin{proof}
If we add attention mask as top $1$ mask, the Eqn of Transformer layer will become:
\begin{equation}
    \mathbf H_\text{out} = \frac{1}{\sqrt{d_\text{out}}}\mathbf H(\mathbf W_v + \mathbf I)
\end{equation}
We use the GCN layer as an example and the message passing can be derived as:
\begin{equation}
\begin{aligned}
    \mathbf H & = \sigma(\mathbf A \mathbf H_\text{out}\mathbf W) \\
      & = \sigma(\mathbf A\frac{1}{\sqrt{d_\text{out}}}\mathbf H(\mathbf W_v + \mathbf I)\mathbf W),
\end{aligned}
\label{equ:message_passing}
\end{equation}
where $\sigma(\cdot)$ is the activation function and if we set $\mathbf W_v$ as diagonal matrix with the diagonal value as $\sqrt{d_\text{out}}-1$ and the \autoref{equ:message_passing} will become:
\begin{equation}
\begin{aligned}
    \mathbf H = \sigma(\mathbf A\mathbf H\mathbf W)
\end{aligned}
\label{equ:new_message_passing}
\end{equation}
Therefore, TransGNN has at least the expression ability of GNN.
\end{proof}

\begin{figure}[h]
\centering
\includegraphics[width=0.40\textwidth]{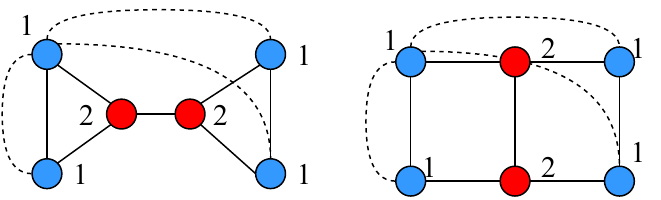}
\caption{Example illustration. The dotted line denotes the receptive field of the Transformer layer, and the color denotes the similarity relationship.}
\label{figure:example}
\end{figure}

\begin{theorem}
TransGNN can be more expressive than 1-WL Test.
\end{theorem}
\begin{proof}
With the help of Transformer layer and positional encoding, TransGNN can aggregate more relevant information and structure information to improve the message passing process, which can be more expressive than 1-WL Test~\cite{douglas2011weisfeiler}. We give an illustration in \autoref{figure:example}. These two graphs cannot be distinguished by 1-WL Test. However, because the transformer layer expand the receptive field and the positional encoding can capture the local structure, they can be distinguish by TransGNN. For example the node in the top left gets the shortest path information \{0,1,3,3\} and \{0,1,2,3\} respectively.
\end{proof}

\section{Experiments}
\par To evaluate the effectiveness of our TransGNN, our experiments are designed to answer the following research questions:
\begin{itemize}[leftmargin=*,noitemsep,topsep=0pt]
    \item \textbf{RQ1}:  Can our proposed TransGNN outperform the state-of the-art baselines of different categories?
    \item \textbf{RQ2}: How do the key components of TransGNN (\textit{e.g.,} the attention sampling, postional encoding, the message passing update) contribute to the overall performance on different datasets?
    \item \textbf{RQ3}: How do different hyper-parameters affect TransGNN?
    \item \textbf{RQ4}: How does TransGNN perform under different GNN layers?
    \item \textbf{RQ5}: How efficient is TransGNN compared to the baselines?
\end{itemize}

\subsection{Experimental Settings}
\subsubsection{Datasets.}

We evaluate the proposed model on five real-world representative datasets (\textit{i.e.,} Yelp, Gowalla, Tmall, Amazon-Book, and MovieLens),  which vary significantly in domains and sparsity:
\begin{itemize}
    \item \textbf{Yelp}: This dataset contains user ratings on business venues collected from Yelp. Following other papers on implicit feedback~\cite{huang2021knowledge}, we treat users’ rated venues as interacted items and treat unrated venues as non-interacted items.
    \item \textbf{Gowalla}: It contains users’ check-in records on geographical locations obtained from Gowalla. This evaluation dataset is generated from the period between 2016 and 2019.
    \item \textbf{Tmall}: This E-commerce dataset is released by Tmall, containing users’ behaviors for online shopping. We collect the page-view interactions during December in 2017.
    \item \textbf{Amazon-Book}: Amazon-review is a widely used dataset for
product recommendation~\cite{he2016ups}. We select Amazon-book from the
collection. Similarly, we use the 10-core setting to ensure that each
user and item have at least ten interactions.
    \item \textbf{MovieLens}: This is a popular benchmark for
evaluating recommendation algorithms~\cite{harper2015movielens}. We adopt the well-established version, MovieLens 10m (\textbf{ML-10m}), which contains about 10 million ratings of 10681 movies by 71567 users. 
\end{itemize}

For dataset preprocessing, we follow the common practice in~\cite{huang2021knowledge,xia2022self,tang2018personalized}. For all datasets, we convert all numeric ratings or the presence of a review to implicit feedback of 1 (\textit{i.e.,} the user interacted with the item). After that, we group the interaction records by users and build the interaction sequence for each user by sorting these interaction records according to the timestamps. To ensure the quality of the dataset, following the common practice~\cite{huang2021knowledge,xia2022self,tang2018personalized,sun2019bert4rec}, we filtered out users and items with too few interactions. The statistics of the datasets are summarized in Table~\ref{table:statustics_dataset}.

\subsubsection{Evaluation Protocols.} 
Following the recent CF models~\cite{he2020lightgcn,wu2021self}, we adopt all-rank evaluation protocol. Under this protocol, during the evaluation of a user, both the positive items from the test set and all the non-interacted items are collectively ranked and assessed. For the assessment of recommendation performance, we have opted for widely recognized metrics, namely, \textbf{Recall@N} and Normalized Discounted Cumulative Gain (\textbf{NDCG@N})~\cite{ren2020sequential,wang2019neural}. The value of N in these metrics is set to 20 and 40.

\begin{table}[t]
\caption{Statistics of datasets after preprocessing}
\label{table:statustics_dataset}
\begin{tabular}{c|c|c|c|c}
\toprule
Datasets & \# Users & \# Items & \# Interactions & Density  \\
\midrule
Yelp  & 29,601  & 24,734     & 1,527,326     & $2.1\times 10^{-3}$     \\
Gowalla & 50,821  & 24,734      & 1,069,128     & $4.0\times 10^{-4}$    \\
Tmall  & 47,939  & 41,390     & 2,357,450      & $1.2\times 10^{-3}$     \\
Amazon-Book   & 78,578  & 77,801   & 3,200,224     & $5.2\times 10^{-4}$    \\
ML-10M   & 69,878  & 10,196  &  9,998,816    & $1.4\times 10^{-2}$   \\
\bottomrule
\end{tabular}
\end{table}

\begin{table*}[t]
\huge
\centering
\caption{Comparison of recommendation performance with baseline models in numerical values. The results of the best performing baseline are underlined. The numbers in bold indicate statistically significant improvement (p < .01) by the pairwise t-test comparisons over the other baselines. $^\star$, $^*$, $^\dagger$ and $^\ddagger$ denote GNN-based, HGNN-based, SSL-enhanced and Transformer-GNN-based models, respectively.}
\label{tab:overall}
\renewcommand{\arraystretch}{1.5}
\resizebox{0.86\linewidth}{!}{
\begin{tabular}{c|l|cccccccccccc|c|c}

\hline
         \multirow{2}*{Dataset} & \multirow{2}*{Metric} & (a) & (b) & (c) & (d) & (e) & (f) & (g) & (h) & (i) & (j) & (k) & (l) & (m) & (n)\\
         \cline{3-16}
         &  & AutoR & GCMC$^\star$ & PinSage$^\star$ & NGCF$^\star$ & LightGCN$^\star$ & GCCF$^\star$ & HyRec$^*$ & DHCF$^*$ & MHCN$^\dagger$ & SLRec$^\dagger$ & SGL$^\dagger$ & SHT$^\ddagger$ & TransGNN & Improv.\\
         \hline
         \hline
    
\multirow{4}{*}{Yelp}     
& Recall@20        & 0.0348 & 0.0462 & 0.0471  & 0.0579	  & 0.0653	 & 0.0462 & 0.0467   & 0.0458 & 0.0646   &  0.0639   &  \underline{0.0675}   &  0.0651   &  \textbf{0.0817} & +21.04\%   \\
                          & Recall@40        & 0.0515 & 0.0836 & 0.0895  & 0.1131  & 0.1203 & 0.0930 & 0.0891   & 0.0851 & 0.1226  &  0.1221   &  \underline{0.1269}   &  0.1091    &  \textbf{0.1506}   &  +18.68\%    \\
                           & NDCG@20        & 0.0236 & 0.0379	 & 0.0393	   & 0.0477  & 0.0532 & 0.0398 & 0.0355   & 0.0376 & 0.0410  &  0.0358   &  \underline{0.0555}   & 0.0546    &  \textbf{0.0734}   &  +32.25\%    \\
                            & NDCG@40        & 0.0358 & 0.0443 & 0.0417   & 0.0693  & 0.0828 & 0.0704 & 0.0628   & 0.0716 & 0.0810  &  0.0806   &  \underline{0.0871} & 0.0709     &  \textbf{0.0999}   &  +14.70\%    \\
\midrule
\multirow{4}{*}{Tmall}   
& Recall@20        & 0.0235 & 0.0292 & 0.0373  & 0.0402  & \underline{0.0525} & 0.0327  & 0.0333   & 0.0385 & 0.0405   &  0.0394  &  0.0436 &0.0387    & \textbf{0.0608}   &  +15.81\%    \\
                          & Recall@40        & 0.0405 & 0.0511 & 0.0594   & 0.0613  & \underline{0.0867} & 0.0619  & 0.0677   & 0.0784 & 0.0783  &  0.0782   &  0.0802 & 0.0645  &  \textbf{0.0989}   &  +14.07\%    \\
                           & NDCG@20        & 0.0120 & 0.0147  & 0.0221  & 0.0294  & \underline{0.0372} & 0.0257 & 0.0262   & 0.0318 & 0.0337  &  0.0354   &  0.0363 & 0.0262  &  \textbf{0.0422}   &  +13.44\%    \\
                            & NDCG@40        & 0.0103 & 0.0200 & 0.0310   & 0.0405  & \underline{0.0508} & 0.0402 & 0.0428   & 0.0417 & 0.0473  &  0.0413  &  0.0505 & 0.0352  &  \textbf{0.0555}   &  +9.25\%    \\
\midrule
\multirow{4}{*}{Gowalla}     
& Recall@20        & 0.1298 & 0.1395 & 0.1380	  & 0.1570	  & \underline{0.1820} & 0.1577 & 0.1649   & 0.1642 & 0.1710   &  0.1656   &  0.1709 & 0.1232    & \textbf{0.1887}   &  +3.68\%    \\
                          & Recall@40        & 0.1359 & 0.1783  & 0.1852   & 0.2270  & \underline{0.2531} & 0.2348 & 0.2333   & 0.2422 & 0.2347  &  0.2331   &  0.2502 & 0.1804     &  \textbf{0.2640}   &  +4.31\%    \\
                           & NDCG@20        & 0.1178 & 0.1204 & 0.1196	  & 0.1327	  & \underline{0.1547} & 0.1285 & 0.1452   & 0.1453 & 0.1510  &  0.1473   &  0.1529 & 0.0731     &  \textbf{0.1602}   &  +3.56\%    \\
                            & NDCG@40        & 0.0963 & 0.1060 & 0.1035   & 0.1102  & \underline{0.1270} & 0.1150 & 0.1162   & 0.1256 & 0.1230  &  0.1213   &  0.1259 & 0.0881   &  \textbf{0.1318}   &  +3.78\%    \\
\midrule
\multirow{4}{*}{Amazon-Book}     & Recall@20        & 0.0287 & 0.0288 & 0.0282	  & 0.0344  & 0.0411 & 0.0415  & 0.0427   & 0.0411 & 0.0552   &  0.0521  &  0.0478 &\underline{0.0740}    & \textbf{0.0801}   &  +8.24\%    \\
                          & Recall@40        & 0.0492 & 0.0539 & 0.0625   & 0.0590  & 0.0741 & 0.0772  & 0.0793   & 0.0824 & 0.0846  &  0.0815   &  0.1023 & \underline{0.1164}  &  \textbf{0.1239}   &  +6.44\%    \\
                           & NDCG@20        & 0.0156 & 0.0224  & 0.0219	  & 0.0263	  & 0.0318 & 0.0308 & 0.0330   & 0.0312 & 0.0384  &  0.0356   &  0.0379 & \underline{0.0553}  &  \textbf{0.0603}   &  +9.04\%    \\
                            & NDCG@40        & 0.0228 & 0.0336 & 0.0392   & 0.0364  & 0.0461 & 0.0440 & 0.0432   & 0.0414 & 0.0492  &  0.0475  &  0.0531 & \underline{0.0690}  &  \textbf{0.0748}   &  +8.41\%    \\
\midrule
\multirow{4}{*}{ML-10M}     & Recall@20        & 0.1932 & 0.2015 & 0.2251  & 0.2136  & 0.2402 & 0.2356  & 0.2371   & 0.2368 & 0.2509   &  0.2415  &  0.2474 &\underline{0.2546}    & \textbf{0.2668}   &  +4.79\%    \\
                          & Recall@40        & 0.2593 & 0.2726 & 0.3050   & 0.3127  & 0.3406 & 0.3321  & 0.3376   & 0.3263 & 0.3424  &  0.3380   &  0.3603 & \underline{0.3794}  &  \textbf{0.3962}   &  +4.43\%    \\
                           & NDCG@20        & 0.1903 & 0.2052  & 0.2359  & 0.2218  & 0.2704 & 0.2682 & 0.2691   & 0.2697 & 0.2761  &  0.2725   &  0.2813 & \underline{0.3038}  &  \textbf{0.3223}   &  +6.09\%    \\
                            & NDCG@40        & 0.2180 & 0.2349 & 0.2452   & 0.2421  & 0.2874 & 0.2802 & 0.2814   & 0.2846 & 0.3007  &  0.2932  &  0.3194 & \underline{0.3384}  &  \textbf{0.3525}   &  +4.17\%    \\                       
\bottomrule
\end{tabular}}
\end{table*}

\subsubsection{Baselines.}
\par We compare TransGNN with five types of baselines: (1) Autoencoder-based method, \textit{i.e.,} AutoR~\cite{sedhain2015autorec}. (2) GNN-based methods, which includes GCMC~\cite{berg2017graph}, PinSage~\cite{ying2018graph}, NGCF~\cite{wang2019neural}, LightGCN~\cite{he2020lightgcn} and GCCF~\cite{chen2020revisiting}. (3) Hypergraph-based methods,  which includes HyRec~\cite{wang2020next} and DHCF~\cite{ji2020dual}. (4) Self-supervised learning enhanced GNN-based methods, which includes MHCN~\cite{yu2021self}, SLRec~\cite{yao2021self} and SGL~\cite{wu2021self}. (5) To verify the effectiveness of our integration of Transformer and GNNs, we also include the Hypergraph Transformer enhanced method,~\textit{i.e.,} SHT~\cite{xia2022self} in comparison. 

\subsubsection{Reproducibility.} 
\par We use three Transformer layers with two GNN layers sandwiched between them. For the transformer layer, multi-head attention is used. For the GNN layer, we use GraphSAGE as the backbone model. We adopt the message passing based attention update in our main experiment. 
We consider attention sampling size $d\in \{5, 10, 15, 20, 25, 30, 35\}$ , number of heads $h \in \{2, 4, 8, 16, 32\}$, dropout ratio $d\in [0, 1]$ and weight decay $d\in [0, 1e^{-2}]$. 
We apply grid search to find the optimal hyper-parameters for each model. We use Adam to optimize our model. We train each model with early stop strategies until the validation recall value does not improve for 20 epochs on a single NVIDIA A100 SXM4 80GB GPU. The average results of five runs are reported. 

\subsection{Overall Performance Comparison (RQ1)}
\par In this section, we validate the effectiveness of our TransGNN framework by conducting the overall performance evaluation on the five datasets and comparing TransGNN with various baselines. 
The results are presented in Table~\ref{tab:overall}. 

Compared with Autoencoder-based methods like AutoR, we observe that GNN-based methods, including TransGNN, exhibit superior performance. This can be largely attributed to the inherent ability of GNNs to adeptly navigate and interpret the complexities of graph-structured data. Autoencoders, while efficient in latent feature extraction, often fall short in capturing the relational dynamics inherent in user-item interactions, a forte of GNNs. When considering hypergraph neural networks (HGNNs) such as HyRec and DHCF, it's apparent that they surpass many GNN-based methods (\textit{e.g.,} GCMC, PinSage, NGCF, STGCN). The key to this enhanced performance lies in their ability to capture high-order and global graph connectivity, a dimension where conventional GNNs often exhibit limitations. This observation underscores the necessity of models that can comprehend more intricate and interconnected graph structures in recommendation systems. TransGNN stands out by integrating the Transformer's strengths, particularly in expanding the receptive field. This integration enables TransGNN to focus on a broader and more relevant set of nodes, thereby unlocking the latent potential of GNNs in global relationship learning. This synthesis is particularly effective in capturing long-range dependencies, which is a notable limitation in standalone GNNs.

In the realm of Self-Supervised Learning (SSL), methods like MHCN, SLRec, and SGL have shown improvements in graph-based Collaborative Filtering models. These advancements are primarily due to the incorporation of augmented learning tasks which introduce beneficial regularization to the parameter learning processes. This strategy effectively mitigates overfitting risks based on the input data itself. However, TransGNN surpasses these SSL baselines, a success we attribute to its global receptive field facilitated by the Transformer architecture. This global perspective enables the adaptive aggregation of information at a larger scale, in contrast to SSL-based methods which are constrained to batch-level sampling, limiting their scope. Moreover, SSL methods often lack the robustness required to effectively tackle data noise. TransGNN, with its attention sampling module, adeptly addresses this challenge by filtering out irrelevant nodes, thereby refining the graph structure and significantly reducing noise influence.

An intriguing aspect of TransGNN's performance is observed under different top-K settings in evaluation metrics like Recall@K and NDCG@K. Notably, TransGNN demonstrates more substantial performance improvements over baseline models when K is smaller. This is particularly relevant considering the position bias in recommendation systems, where users are more inclined to focus on higher-positioned items in recommendation lists. TransGNN's efficacy in these scenarios suggests it is well-suited for generating user-friendly recommendations that align closely with user preferences, especially at the top of the recommendation list.

\begin{table}[t]
\centering
\Large
  \caption{Ablation studies on different components.}
  \label{table:ablation_study}
  \resizebox{\linewidth}{!}{
  \begin{tabular}{c|c|c c|c c|c c}
    \hline
    \multirow{2}*{Category} & Data & \multicolumn{2}{c|}{\textit{Yelp}} & \multicolumn{2}{c|}{\textit{Gowalla}} & \multicolumn{2}{c}{\textit{Tmall}} \\ \cline{2-8}
              ~ & Variants & Recall & NDCG  & Recall & NDCG & Recall & NDCG \\
    \hline
    Attention Sampling & -AS & 0.0686   & 0.0583   & 0.1492   & 0.0961   &  0.0418  &  0.0349  \\ \hline
    \multirow{4}*{Positional Encoding} & -SPE & 0.0910  & 0.0760  & 0.1836  & 0.1104  & 0.0591  & 0.0412  \\
        ~ & -DE &  0.0922 & 0.0781  & 0.1873  & 0.1109  & 0.0602  & 0.0416  \\
        ~ & -PRE & 0.0907  & 0.0772  & 0.1845  & 0.1106  & 0.0595  &  0.0417 \\\cline{2-8}
        ~ & -PE & 0.0704  & 0.0610  & 0.1609  & 0.1013  & 0.0437  & 0.0375  \\ \hline
    \multirow{2}*{TransGNN} & -Trans & 0.0479  & 0.0391  & 0.0978  & 0.0556  & 0.0237  & 0.0159  \\
    ~ & -GNN & 0.0390  & 0.0316  & 0.0574  & 0.0341  & 0.0209  & 0.0126  \\ \hline
    \multirow{2}*{Attention Update}  & +RW & 0.0924  & 0.0782  & 0.1880  & 0.1109  &  0.0601 & 0.0417  \\
    ~ & -MP & 0.0723  & 0.0611  & 0.1550  & 0.1007  & 0.0462  & 0.0360  \\ \hline
    Original & TransGNN & \textbf{0.0927}  & \textbf{0.0787}  & \textbf{0.1887}  & \textbf{0.1121}  & \textbf{0.0608}  & \textbf{0.0422}  \\ \hline
  \end{tabular}
  }
\end{table}

\subsection{Ablation Study (RQ2)}

\par To evaluate the effectiveness of the proposed modules in TransGNN, we systematically remove key techniques within its four primary components: the attention sampling module (-AS), the positional encoding module (-PE), the core TransGNN module (-Trans and -GNN), and the attention update module (-MP). Additionally, we ablate specific elements within the positional encoding and attention update modules by removing shortest-path-hop-based (-SPE), degree-based (-DE), and PageRank-based (-PRE) encodings, and replacing the message passing attention update with a random walk-based update (+RW).
We use GraphSAGE as the backbone GNN to report the performance when ignoring different components. All the settings are the same as the default. The variants are re-trained for test on Yelp, Gowalla and Tmall datasets. From Table~\ref{table:ablation_study}, we have the following major conclusions:

\begin{itemize}[leftmargin=*,noitemsep,topsep=0pt]
    \item 
    Removing the attention sampling module (-AS) results in a noticeable performance decline across all datasets, highlighting the need for focused attention to effectively manage the complex user-item interaction space. This module is crucial for filtering out irrelevant nodes within a global attention context. Without it, the model's node selection becomes less precise, resulting in less targeted recommendations. 
    \item The exclusion of the positional encoding module (-PE) leads to compromised results, indicating that the Transformer layer alone cannot capture the structure information adequately. This is further substantiated by the reduced performance observed when individual components of the positional encoding—shortest-path-hop-based (-SPE), degree-based (-DE), and PageRank-based (-PRE)—are separately removed. Each of these encodings contributes uniquely to the model’s understanding of the graph topology, reflecting user-item proximity, interaction frequency, and structural importance, respectively. 
    \item The simultaneous necessity of both the Transformer and GNN layers is clearly demonstrated when either is removed (-Trans and -GNN). The significant decline in performance underscores the synergistic relationship between these two layers. The Transformer layer, with its expansive receptive field, brings a global perspective to the table, while the GNN layer contributes a comprehensive understanding of graph topology. Their combined operation is crucial for a holistic approach to user modeling in TransGNN, blending global and local insights.
    \item In examining the sampling update strategies, it is observed that the random walk-based update (+RW) does not perform as well as the message passing-based update. This could be attributed to the inherent noise in the edges, which might lead the random walk to less relevant nodes, highlighting the superiority of a more structured update mechanism.
    \item Finally, the lack of the message passing update strategy (-MP) leads to a performance decline, indicating that static attention samples may become outdated during training. The dynamic nature of message passing ensures that attention samples stay current with evolving user-item interactions, which is essential for maintaining the accuracy and relevance of TransGNN's recommendations, allowing it to adapt to changes in user preferences and item attributes.
\end{itemize}


\begin{figure}[h]
\centering
\subfigure[Sampling size influence.]{
\includegraphics[width=0.25\textwidth]{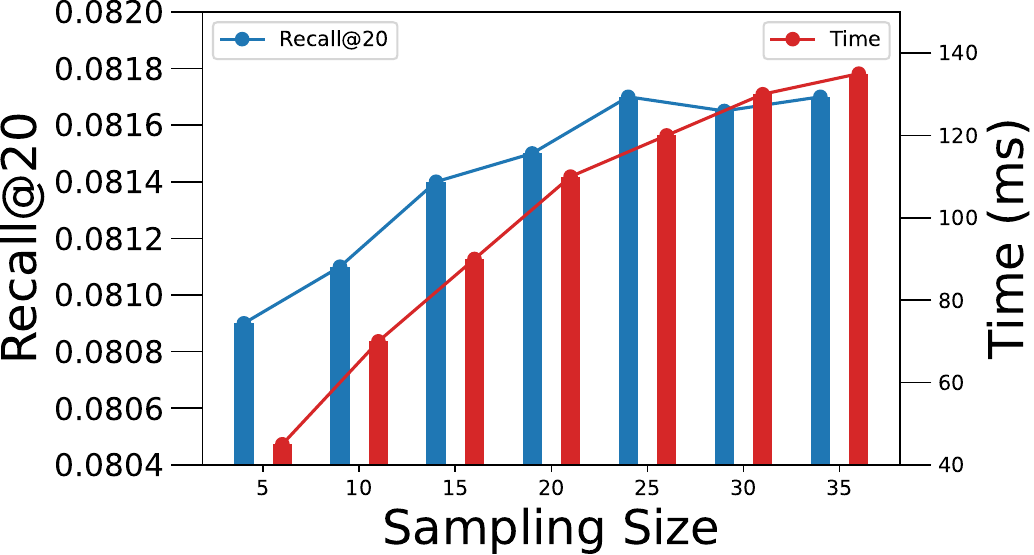}
\label{figure:sampling_size_1}
}
\subfigure[Best size of different data scale.]{
\includegraphics[width=0.20\textwidth]{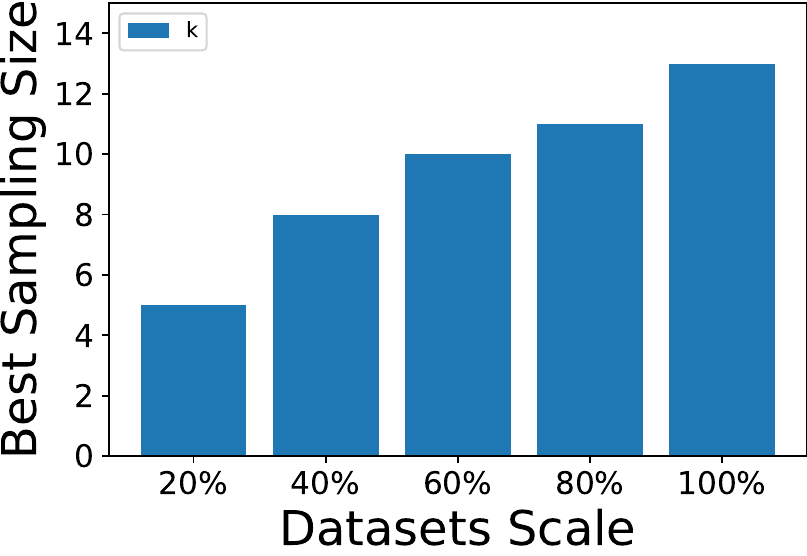}
\label{figure:sampling_size_2}
}
\caption{Different sampling size influence and best sampling size for different graph scale.}
\label{figure:sampling_size}
\end{figure}

\subsection{Attention Sampling Study (RQ3)}

\par We study the influence of attention sampling size using Yelp as an example. The results can be found in \autoref{figure:sampling_size}. We have the following observations:
\begin{itemize}[leftmargin=*,noitemsep,topsep=0pt]
    \item  In Figure \ref{figure:sampling_size_1}, as $k$ increases from 5 to 35, a trend is observed where fewer nodes reduce computational load but degrade performance. This decline is due to insufficient nodes capturing essential user-item connections, impacting the model’s accuracy in recommendations. Conversely, overly extensive node sampling fails to proportionally enhance performance, indicating a saturation point where extra nodes add computational strain without improving predictive accuracy.
    \item 
    In Figure \ref{figure:sampling_size_2}, we adjust the graph scale to identify the optimal sampling size. Interestingly, we find that a relatively small sampling size is found to be sufficient for achieving good performance. 
    This finding suggests that TransGNN can operate efficiently without necessitating extensive computational resources. 
    The ability to extract meaningful insights from a smaller subset of nodes underscores the effectiveness in identifying the most pertinent information from the user-item interaction graph.
    \item Furthermore, when the graph size is varied, we observe that only a slight increase in sample numbers is needed as the graph expands, indicating TransGNN's scalability. Despite increasing graph complexity, the model does not demand a large increase in resources to sustain performance. This scalability is essential for practical applications with large and varied dataset sizes.
\end{itemize}

\begin{table}[t]
\centering
\Large
  \caption{Performance comparison between SHT and TransGNN at different layers.}
  \label{table:oversmoothing}
  \resizebox{0.7\linewidth}{!}{
  \begin{tabular}{c|c|c c}
    \hline
    \multirow{2}*{\textbf{Layer \#}} & \multirow{2}*{\textbf{Method}} & \multicolumn{2}{c}{\textbf{Amazon-Book}} \\ \cline{3-4}
              ~ & & \textbf{Recall@20} & \textbf{NDCG@20}  \\
    \hline
    \multirow{2}*{\textbf{1 Layer}} & SHT & 0.0712  & 0.0497   \\
        ~ & TransGNN &  0.0723 & 0.0502  \\\hline
    \multirow{2}*{\textbf{2 Layer}} & SHT & 0.0740  & 0.0553   \\
        ~ & TransGNN &  0.0772 & 0.0565  \\\hline
    \multirow{2}*{\textbf{3 Layer}} & SHT & 0.0725  & 0.0541   \\
        ~ & TransGNN &  0.0801 & 0.0603  \\\hline
    \multirow{2}*{\textbf{4 Layer}} & SHT & 0.0713  & 0.0513   \\
        ~ & TransGNN &  0.0805 & 0.0606  \\\hline
    \multirow{2}*{\textbf{5 Layer}} & SHT & 0.0684  & 0.0467   \\
        ~ & TransGNN &  0.0807 & 0.0609  \\\hline
  \end{tabular}
  }
\end{table}

\subsection{ Study on the number of GNN Layers (RQ4)}

We evaluate how varying the number of GNN layers affects performance. Table~\ref{table:oversmoothing} shows that SHT's performance improves with up to two layers but deteriorates with additional layers, indicating over-smoothing in deeper architectures. Conversely, TransGNN demonstrates enhanced performance as layer count increases, reflecting its ability to counteract over-smoothing and effectively capture wider graph dependencies. This contrast highlights TransGNN's  capability in managing the complexities of deeper graph networks, affirming its robustness for recommendation systems.

\subsection{Complexity and Efficiency Analysis (RQ5)}

 We conduct an analysis on TransGNN's execution efficacy, as delineated in Table~\ref{tab:complexity}. Comparing with GNN-based (\textit{i.e.,} NGCF), Hypergraph-based (\textit{i.e.,} HyRec) and Transformer-GNN-based baselines (\textit{i.e., } SHT) , we have the following observations: 
\begin{itemize}[leftmargin=*,noitemsep,topsep=0pt]
    \item TransGNN requires less GPU memory due to its node sampling strategy in the Transformer component, which avoids full-graph attention calculations. By focusing only on the most relevant nodes, TransGNN reduces the computational load typical in large graph processing. This method conserves memory and enhances the model's agility and adaptability to different dataset sizes.
    \item TransGNN has better efficiency in
both training ($30\%~40\%$ less time) and inference ($30\%-40\%$ less time)
compared with NGCF. 
This substantial reduction in processing time renders TransGNN an especially practical choice for scenarios where rapid response is crucial, such as real-time recommendation systems or applications with frequent model updates. The improved performance ensures that systems can operate more smoothly and efficiently, adapting more quickly to new user interactions.
    \item Furthermore,  
    TransGNN achieves comparable, if not superior, training and inference times relative to DHCF and SHT, while simultaneously maintaining a smaller memory footprint. This balance between speed and resource utilization is crucial in scenarios where computational resources are limited. The lower memory requirement, quantified at $20\%~30\%$ less than the other models, underscores TransGNN's suitability for deployment in environments with constrained computational resources.
    \item 
    Additionally, TransGNN's minimal Floating Point Operations (FLOPs) compared to all baselines is a testament to its computational efficiency. This aspect is particularly important for deployment on low-resource devices, where managing computational overhead is critical. The lower FLOPs indicate that TransGNN requires fewer computational resources to perform the same tasks as its counterparts, which is a significant advantage in resource-constrained environments.
\end{itemize}

\begin{table}[t]
    \caption{Parameters number and execution efficiency analysis of models on Yelp dataset.} 
    \label{tab:complexity}
    \centering
    \resizebox{\linewidth}{!}{
    \begin{tabular}{l|c|c|c|c}
    \hline
    Yelp & TransGNN & NGCF & HyRec & SHT  \\ 
    \hline
    \hline
     \# Parameters & 1.77M & 2.30M & 1.91M & 1.75M\\ 
     \# GPU Memory & 1.79GB & 0.62GB & 2.32GB  &  2.33GB\\
     \# FLOPs ($\times10^6$) & 3.74 & 8.41 & 5.36 & 4.89\\ 
     Training Cost (per epoch) & 4.75s & 7.88s & 3.65s & 3.97s\\
     Inference Cost (per epoch)  & 16.95s & 27.07s & 12.48s & 12.57s \\
     \hline
     Recall@20 & \textbf{0.0927}  & 0.0294  & 0.0472 & 0.0651 \\
     \hline
    \end{tabular}
    }
\end{table}

\section{Conclusion}
\par In this paper, we propose TransGNN to help GNN expand its receptive field with low overhead. We first use three kinds of positional encoding to capture the structure information for the transformer layer. Then the Transformer layer and GNN layer are used alternately to make every node focus on the most relevant samples. Two efficient sample update strategies are proposed for medium- and large-scale graphs, which reduce the computational cost. At last, we conduct experiments on five datasets to show the effectiveness of TransGNN compared to the state-of-the-art baselines.

\section*{Acknowledgements}
This work was supported by the Natural Science Foundation of China (No. 62372057, 62272200, 62172443, U22A2095).

\bibliographystyle{ACM-Reference-Format}
\balance
\bibliography{citation}

\clearpage

\end{document}